\documentclass{article}

\usepackage{microtype}
\usepackage{graphicx}
\usepackage{subfigure}
\usepackage{booktabs} % for professional tables

\usepackage{caption}

\usepackage{hyperref}
\usepackage[table]{xcolor} 
\definecolor{grey}{rgb}{0.5,0.5,0.5}  % 中灰
\usepackage[preprint]{icml2026}
\usepackage{xurl}
\usepackage{amsmath}
\usepackage{amssymb}
\usepackage{mathtools}
\usepackage{amsthm}
\usepackage{bm,amsfonts}
\usepackage{enumitem}
\usepackage[capitalize,noabbrev]{cleveref}
\usepackage{tabularx} 
\theoremstyle{plain}
\newtheorem{theorem}{Theorem}[section]
\newtheorem{proposition}[theorem]{Proposition}

\theoremstyle{definition}
\newtheorem{definition}[theorem]{Definition}

\theoremstyle{remark}

\usepackage[textsize=tiny]{todonotes}
\newcolumntype{Y}{>{\centering\arraybackslash}X}

\icmltitlerunning{Causal Autoregressive Diffusion Language Model}

\begin{document}

\twocolumn[
\icmltitle{Causal Autoregressive Diffusion Language Model}

% It is OKAY to include author information, even for blind
% submissions: the style file will automatically remove it for you
% unless you've provided the [accepted] option to the icml2025
% package.

% List of affiliations: The first argument should be a (short)
% identifier you will use later to specify author affiliations
% Academic affiliations should list Department, University, City, Region, Country
% Industry affiliations should list Company, City, Region, Country

% You can specify symbols, otherwise they are numbered in order.
% Ideally, you should not use this facility. Affiliations will be numbered
% in order of appearance and this is the preferred way.
\icmlsetsymbol{equal}{*}

\begin{icmlauthorlist}
\icmlauthor{Junhao Ruan}{neu,meituan}
\icmlauthor{Bei Li}{meituan}
\icmlauthor{Yongjing Yin}{meituan}
\icmlauthor{Pengcheng Huang}{neu}
\icmlauthor{Xin Chen}{meituan}
\icmlauthor{Jingang Wang}{meituan}
\icmlauthor{Xunliang Cai}{meituan}
\icmlauthor{Tong Xiao}{neu,niutrans}
\icmlauthor{Jingbo Zhu}{neu,niutrans}
\end{icmlauthorlist}

\icmlaffiliation{meituan}{Meituan Inc.}
% \icmlaffiliation{nlplab}{NLP Lab, School of Computer Science and Engineering, Northeastern University, Shenyang, China}
\icmlaffiliation{neu}{School of Computer Science and Engineering, Northeastern University, Shenyang, China}
\icmlaffiliation{niutrans}{NiuTrans Research, Shenyang, China}
% \icmlcorrespondingauthor{Junhao Ruan}{rangehow@outlook.com}
\icmlcorrespondingauthor{Bei Li}{libei17@meituan.com}
\icmlcorrespondingauthor{Tong Xiao}{xiaotong@mail.neu.edu.cn}

\icmlkeywords{Machine Learning, ICML}

\vskip 0.3in
]

% this must go after the closing bracket ] following \twocolumn[ ...

% This command actually creates the footnote in the first column
% listing the affiliations and the copyright notice.
% The command takes one argument, which is text to display at the start of the footnote.
% The \icmlEqualContribution command is standard text for equal contribution.
% Remove it (just {}) if you do not need this facility.

\printAffiliationsAndNotice{}  % leave blank if no need to mention equal contribution
% \printAffiliationsAndNotice{\icmlEqualContribution} % otherwise use the standard text.
% \printAffiliationsAndNotice{ \textsuperscript{†}Corresponding author.}
\begin{abstract}

In this work, we propose Causal Autoregressive Diffusion (CARD), a novel framework that unifies the training efficiency of ARMs with the high-throughput inference of diffusion models. CARD reformulates the diffusion process within a strictly causal attention mask, enabling dense, per-token supervision in a single forward pass. To address the optimization instability of causal diffusion, we introduce a soft-tailed masking schema to preserve local context and a context-aware reweighting mechanism derived from signal-to-noise principles. This design enables dynamic parallel decoding, where the model leverages KV-caching to adaptively generate variable-length token sequences based on confidence. Empirically, CARD outperforms existing discrete diffusion baselines while reducing training latency by 3 $\times$ compared to block diffusion methods. Our results demonstrate that CARD achieves ARM-level data efficiency while unlocking the latency benefits of parallel generation, establishing a robust paradigm for next-generation efficient LLMs.

\end{abstract}

\begin{figure}[t]
    \centering
    \includegraphics[width=\linewidth]{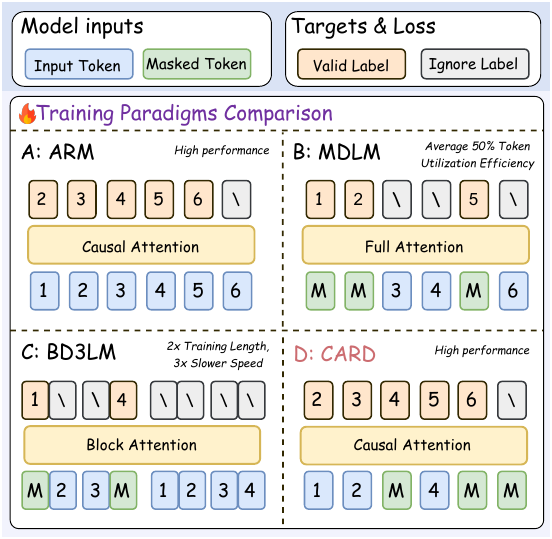}
    \caption{Comparison of training paradigms. Current diffusion methods like MDLM and BD3LM are inefficient compared to ARM; MDLM reaches only 50\% of ARM's expected efficiency, while BD3LM relies on complex masking and sequence duplication. CARD overcomes these issues by using causal diffusion, maintaining the same high efficiency as ARM while achieving better performance.}
    \label{fig:intro}
\end{figure}

\begin{figure*}[t] 
  \centering
  \includegraphics[width=\textwidth]{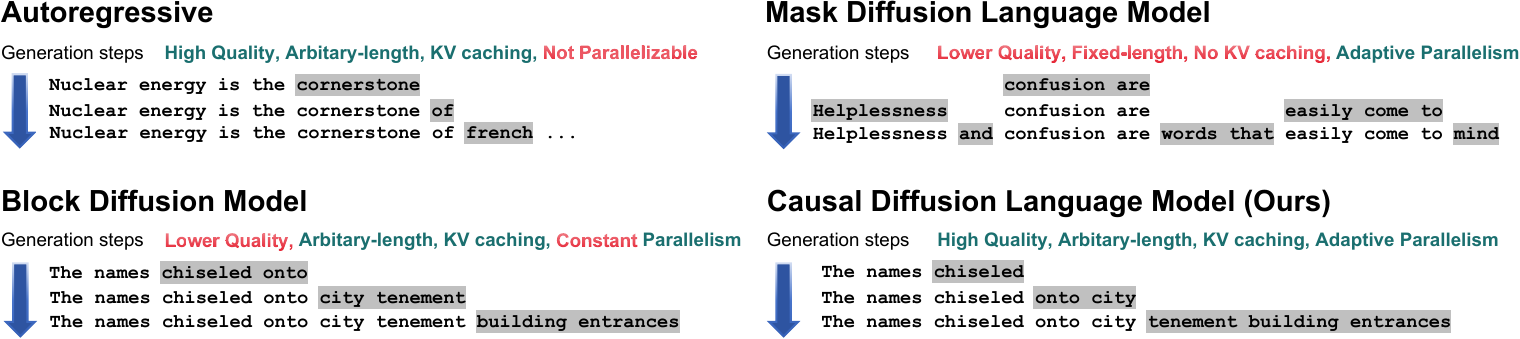}
  \caption{Inference comparison of the four paradigms. CARD achieves high-quality results similar to ARM. With KV cache support, friendly operators, and parallel generation, it offers faster throughput than earlier methods. In particular, our inference parallelism is flexible, unlike BD3LM which is tied to the fixed block size used during training.}
  \label{fig:intro_inference}
\end{figure*}

\section{Introduction}
Causal Autoregressive Models (ARMs) currently serve as the dominant paradigm for training Large Language Models (LLMs), owing to their stable training dynamics and predictable scaling laws. However, as model parameters and test-time compute requirements grow, the sequential nature of autoregressive decoding has emerged as a critical bottleneck. This inefficiency has sparked renewed interest in Text Diffusion Models, which offer theoretical advantages including parallel inference~\cite{d3pm}, iterative refinement~\cite{remdm}, and potentially higher data modeling capacity~\cite{ni2025diffusionlanguagemodelssuper}.

Early attempts at discrete diffusion faced significant hurdles due to complex training objectives involving variational bounds and numerical instabilities caused by noise sampling~\cite{d3pm}. A turning point occurred with the introduction of Simplified Masked Discrete Diffusion Models (MDLM)~\cite{md4,mdlm}. By simplifying the diffusion process into a subspace assumption analogous to a randomized Masked Language Modeling (MLM)~\cite{bert} task, MDLM ushered in the era of scalable text diffusion~\cite{smdm}, enabling the training of modern LLM-scale diffusion models like LLaDA~\cite{llada} and Dream~\cite{dream}.

Despite these advancements, standard MDLMs face severe architectural constraints. As illustrated in the MDLM panel of Figure~\ref{fig:intro}, its reliance on bidirectional (``Full'') attention prevents the utilization of Key-Value (KV) caching. Consequently, inference speed often falls behind ARMs in practical scenarios~\cite{fastdllm}. Furthermore, the arbitrary dependency order in training can lead to ineffective learning pathways \cite{kim2025train}, and the architecture fundamentally lacks support for variable-length generation.

To address these limitations, recent works have proposed hybrid architectures such as Block Diffusion (e.g., BD3LM)~\cite{arriola2025block}. These models (drawn in Figure \ref{fig:intro}) operate at a coarser granularity, applying causal attention between fixed-size blocks and bidirectional attention within them. However, it introduces significant computational overhead. The vectorization required for block-wise training necessitates complex attention masking and can increase memory consumption and training latency by factors of $2\times$  and 3$\times$, respectively. Moreover, the rigid, fixed block size fails to adapt to the varying information density inherent in natural language, limiting dynamic parallelism.

In this work, we propose CARD, a framework that combines the training efficiency of ARMs with the parallel inference of diffusion models through a strictly causal formulation. For training, CARD employs a \textit{shifted causal attention mechanism} where each position predicts its original token from the preceding noised context. This generates a dense diffusion loss for the entire sequence in a single forward pass, achieving 100\% token utilization without the overhead of block vectorization. For inference, CARD's causal structure enables KV-caching (Figure~\ref{fig:intro_inference}), allowing the model to append a variable number of \texttt{[MASK]} tokens to the prefix and decode them in parallel through iterative denoising. This dynamic strategy generates multiple tokens per step when confidence is high while falling back to sequential decoding when necessary.

We empirically validate CARD on 1B-parameter models trained on 300B tokens, benchmarking against state-of-the-art autoregressive and diffusion baselines. Our results demonstrate that CARD effectively bridges the gap between efficiency and performance:
\begin{itemize}[leftmargin=*]
    \item \textbf{Superior Performance:} CARD achieves an average zero-shot accuracy of \textbf{53.2\%}, outperforming existing diffusion models (MDLM and BD3LM) by over \textbf{5.7 points} and matching the generation quality of ARMs. Notably, it achieves the lowest zero-shot perplexity on 6 out of 8 evaluated domains.
    \item \textbf{Training \& Inference Efficiency:} By eliminating block-wise overhead, CARD reduces training latency by \textbf{3$\times$} compared to Block Diffusion, matching the throughput of standard ARMs. During inference, our confidence-based decoding achieves \textbf{1.7$\times$ to 4.0$\times$} wall-clock speedup with negligible quality degradation.
    \item \textbf{Data Potential:} Scaling analysis reveals that CARD possesses higher data efficiency than ARMs in data-constrained settings, continuing to improve performance through repeated training epochs where autoregressive baselines saturate.
\end{itemize}

\section{Background}
\label{sec:background}

We review the evolution of text diffusion models and the specific discrete objective function that serves as the foundation for our work.

\subsection{Evolution of Text Diffusion Models}
Applying diffusion to the discrete domain of language has followed two primary trajectories: continuous embedding methods and discrete state-space models. Continuous approaches, such as Diffusion-LM~\cite{diffusion-lm} and DiffuSeq~\cite{gong2023diffuseq}, map discrete tokens to Gaussian latent spaces. The disconnect between the continuous diffusion process and the discrete nature of text leads to rounding errors during decoding, often resulting in lower generation performance compared to autoregressive baselines.

Discrete DDPM (D3PM)~\cite{d3pm} addressed this by defining the corruption process directly on the vocabulary via transition matrices. While theoretically rigorous, D3PMs initially suffered from optimization instability and inefficient inference. To mitigate this, SEDD~\cite{sedd} reformulated the objective using score entropy, aligning discrete diffusion closer to its continuous counterparts. However, SEDD relied on time-dependent probability ratios, which prevented step-skipping and slowed inference. RADD~\cite{radd} later demonstrated that the explicit time dependency in the input was not strictly necessary for mathematical validity, enabling flexible sampling strategies.

A paradigm shift occurred with the introduction of MDLM~\cite{mdlm} and MD4~\cite{md4}. By isolating the absorbing state (masking) transition, these works reduced the complex variational bound to a simplified, randomized Masked Language Modeling (MLM) objective. This simplification significantly improved numerical stability and allowed for scaling laws to be established~\cite{smdm}, culminating in large-scale pre-trained models like LLaDA~\cite{llada}.

Despite these successes, standard MDLMs utilize bidirectional attention, which prevents the use of KV caching and degrades inference speed for long sequences. BD3LM~\cite{arriola2025block} attempts to bridge this gap by segmenting sequences into fixed-size blocks with causal masking between them. While this restores some parallel generation capabilities, BD3LM imposes significant training overheads due to complex attention masks and input duplication. 
Semi-autoregressive architectures have been further explored in works like LLaDA2~\cite{llada2}, SDAR~\cite{sdar}.

The concurrent WeDLM~\cite{wedlm} further specializes this by employing unidirectional attention within blocks; however, it generally adheres to the block diffusion paradigm where training operates at the block level rather than the token level which will also bring extra training cost. Distinctly, another concurrent work, C$^2$DLM~\cite{c2dlm}, explores causality through the lens of semantic concepts rather than model architecture. It analyzes causal relationships within training data but retains a bidirectional backbone and the standard MDLM training objective.

\begin{figure*}[t!]
    \centering

    % ---  Algorithm ---
    \begin{minipage}[t]{0.49\linewidth}
        \vspace{-10pt}
       \begin{algorithm}[H]
        \caption{CARD Training Framework}
        \label{alg:card_training}
        \footnotesize
        \begin{algorithmic}[1]
            \STATE {\bfseries Input:} Sequence $\mathbf{x}_0$, Model $\theta$
            \STATE {\bfseries Params:} Tail factor $\lambda$, Base $\beta$, Decay $p$
            
            \STATE \textit{// 1. Noise Scheduling}
            \STATE Sample $t \sim \mathcal{U}[0, 1]$
            
            \STATE \textit{// 2. Soft Tail Masking}
            \STATE $N = \max(1, \lfloor L \cdot t \rfloor)$,\quad $W = \min(L, \lfloor N \cdot \lambda \rfloor)$ 
            \STATE Define tail window indices: $\mathcal{I}_{\text{win}} = \{ L - W + 1, \dots, L \}$
            and sample a subset of indices $\mathcal{M} \subset \mathcal{I}_{\text{win}}$ such that $|\mathcal{M}| = N$
            \STATE Initialize $\mathbf{x}^t = \mathbf{x}_0$
            \FOR{each $n \in \mathcal{M}$}
                \STATE $x_n^t \leftarrow \text{[MASK]}$
            \ENDFOR
    
            \STATE \textit{// 3. Context-aware Reweighting}
            \FOR{$n=1$ {\bfseries to} $L$}
                \STATE $C_n = \mathbb{I}[x_n^t \text{ is [MASK]}] \cdot (1 + \mathbb{I}[x_{n-1}^t \text{ is [MASK]}])$
                \STATE $S_n^{\text{local}} = \sum_{i=1}^{n} C_i \cdot (1-p)^{(n-1-i)}$ 
                \STATE $w_n = (\beta + S_n^{\text{local}})^{-1}$
            \ENDFOR
    
            \STATE \textit{// 4. Optimization}
            \STATE $\mathcal{L}_{\text{CARD}} = \sum_{n=1}^{L} w_n \log p_\theta(x_{0,n} \mid \mathbf{x}^t_{<n})$
            \STATE Update $\theta$ using $\nabla_\theta \mathcal{L}_{\text{CARD}}$
        \end{algorithmic}
    \end{algorithm}

    \end{minipage}
    \hfill
    % --- Left Side: Image ---
    \begin{minipage}[t]{0.49\linewidth} 
        \vspace{0pt} 
        \centering
        \includegraphics[width=\linewidth]{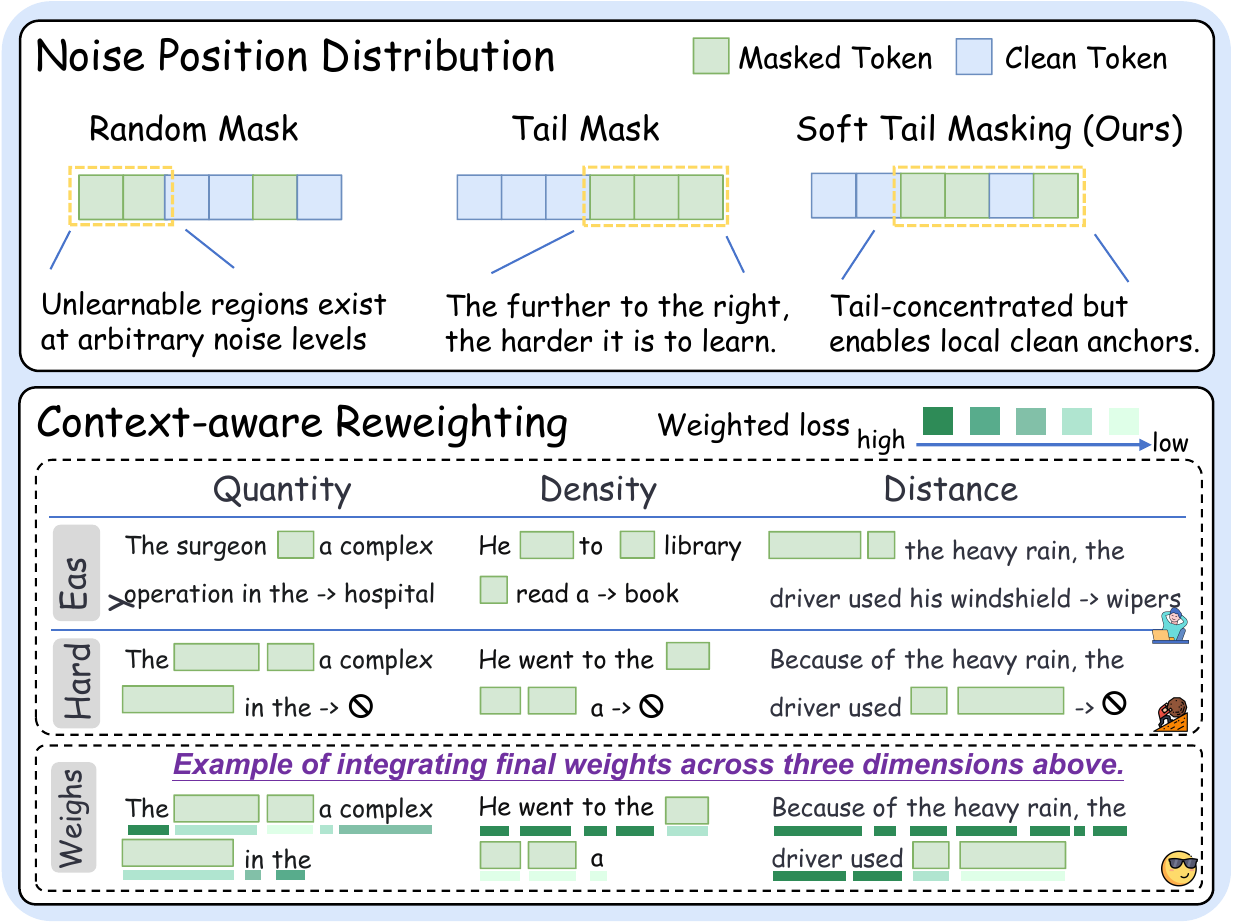}
        \caption{Soft Tail Masking concentrates noise at the sequence tail to resolve unlearnable regions in causal models via local clean anchors. (Bottom) Context-aware Reweighting adaptively down-weights the loss for high-ambiguity contexts similar to the diffusion ELBO principle, improving training stability.}
        \label{fig:CARD_DETAIL}
    \end{minipage}

\end{figure*}
\subsection{Discrete Diffusion Formulation}
Our method builds upon the absorbing state diffusion framework. Let $\textbf{x}_0$ be a sequence of length $L$. D3PM optimizes the negative Variational Lower Bound (ELBO) over $T$ steps, which decomposes into:
\begin{equation}
\label{eq:d3pm_elbo}
\begin{aligned}
    L_{\text{vb}} &= \underbrace{D_{\text{KL}}[q(\textbf{x}_T|\textbf{x}_0) || p(\textbf{x}_T)]}_{L_T} - \underbrace{\mathbb{E}_{q}[\log p_\theta(\textbf{x}_0|\textbf{x}_1)]}_{L_0} \\
    &+ \sum_{t=2}^T \underbrace{\mathbb{E}_{q}[\, D_{\text{KL}}[q(\textbf{x}_{t-1}|\textbf{x}_t, \textbf{x}_0) || p_\theta(\textbf{x}_{t-1}|\textbf{x}_t)] \,]}_{L_{t-1}}.
\end{aligned}
\end{equation}

In practice, to improve training stability and sample quality, D3PM often incorporates an auxiliary cross-entropy loss to directly predict $\textbf{x}_0$:
\begin{equation}
\label{eq:d3pm_total}
    L_{\text{D3PM}} = L_{\text{vb}} + \lambda \mathbb{E}_{q(\textbf{x}_0)} \mathbb{E}_{q(\textbf{x}_t|\textbf{x}_0)} [-\log \tilde{p}_\theta(\textbf{x}_0|\textbf{x}_t)],
\end{equation}
where $\lambda$ is a hyperparameter balancing the two terms. 

This objective involves summing over the entire vocabulary for the posterior computation, making it computationally expensive. MDLM drastically simplifies this by employing a SUBS parameterization~\cite{mdlm}, where the model predicts $\textbf{x}_0$ directly and unmasked tokens are carried over deterministically. The KL divergences collapse, and in the continuous-time limit ($T \to \infty$), the loss becomes a weighted MLM objective:
\begin{equation}
\label{eq:mdlm_nelbo}
\mathcal{L}_{\text{MDLM}} = \mathbb{E}_{t \sim \mathcal{U}[0,1]} \left[ w(t) \sum_{\ell \in \mathcal{M}_t} \log p_\theta(x^\ell | \textbf{x}_t, t) \right].
\end{equation}
Here, the loss is computed only over the masked tokens $\mathcal{M}_t$. The weighting term $w(t) = \frac{\alpha'_t}{1 - \alpha_t}$ is determined by the noise schedule $\alpha_t$. 

BD3LM~\cite{arriola2025block} extends this formulation to interpolate between autoregression and diffusion. By partitioning the sequence $\mathbf{x}$ into $B$ blocks, BD3LM defines an autoregressive distribution over blocks while performing discrete diffusion \emph{within} each block. The objective applies the MDLM loss per block, conditioned on the clean history of previous blocks $\mathbf{x}^{<b}$:
\begin{equation}
\label{eq:bd3lm_loss}
\mathcal{L}_{\text{BD3LM}}(\mathbf{x}; \theta) = \sum_{b=1}^B \mathbb{E}_{t \sim \mathcal{U}[0,1]} \mathbb{E}_{q} \left[ w(t) \log p_\theta(\mathbf{x}^b | \mathbf{x}^b_t, \mathbf{x}^{<b}) \right],
\end{equation}
where $\mathbf{x}^b_t$ represents the noisy state of block at time $t$. 

\subsection{Absorbing State Diffusion Process}
\label{subsec:preliminaries}
In this work, we focus on the specific discrete diffusion process that serves as our foundation. Unlike continuous diffusion models that operate on Gaussian noise, text diffusion models typically define a corruption process over a discrete vocabulary. We consider a continuous-time variable $t \in [0, 1]$, where $t=0$ corresponds to the clean ground-truth sequence $\mathbf{x}_0$, and $t=1$ represents a fully masked sequence.

We utilize the absorbing state (masking) transition. For any token in the sequence at time $t$, the forward process determines whether it remains its original value or transitions to a special \texttt{[MASK]} token. This is governed by a noise schedule $\sigma(t)$. For a given $t$, each token is independently replaced by \texttt{[MASK]} with probability $P(x^t = \text{\texttt{[MASK]}} | x_0) = \sigma(t)$. Throughout this paper, we adopt a linear schedule where $\sigma(t) = t$.

This formulation allows us to bridge the gap between deterministic text and stochastic training. At any step $t$, the model is presented with a partially corrupted version of the input, denoted as $\mathbf{x}^t$. The training objective is to learn a denoising function that recovers the original tokens $\mathbf{x}_0$ from these noisy observations. By sampling $t$ uniformly during training, the model learns to handle varying levels of corruption, from simple text completion to complex generation from scratch.

\section{The CARD Framework}
\label{sec:method}

We propose the Causal Autoregressive Diffusion (CARD) framework, the overall training procedure of which is summarized in Algorithm \ref{alg:card_training}. CARD utilizes a continuous-time noise addition method to apply diffusion processes within a causal architecture. This approach allows the model to leverage the robustness of diffusion training while maintaining the efficiency of autoregressive generation.

\subsection{Synthesizing Autoregression and Diffusion}
\label{subsec:motivation}

The core philosophy of CARD is to unify the stable training dynamics of ARMs with the flexible generation capabilities of Diffusion Models. We achieve this synthesis via a \textit{shifted causal attention mechanism}. Unlike standard ARMs that condition on a static, clean history to model $p(x_n | \mathbf{x}_{<n})$, CARD predicts the original token $x_n$ conditioned on a corrupted prefix $\mathbf{x}_{<n}^t$ sampled from a continuous-time diffusion process.
This architecture allows us to strictly maintain the triangular attention mask inherent to GPT-style models for computational efficiency, while simultaneously minimizing the expected reconstruction error across varying noise intensities. 
We formally define the resulting optimization objective, which aggregates the weighted log-likelihoods across all token positions, as follows:
\begin{equation}
\label{eq:card_loss_general}
\begin{aligned}
    \mathcal{L}_{\text{CARD}} = &\mathbb{E}_{t \sim \mathcal{U}[0,1], \mathbf{x}^t \sim q(\mathbf{x}^t | \mathbf{x}_0)} \\
    &\left[ \sum_{n=1}^{L} w(n, \mathbf{x}_{<n}^t) \log p_\theta(x_n | \mathbf{x}^t_{<n}) \right].
\end{aligned}
\end{equation}
This formulation generates dense supervision for the entire sequence in a single forward pass, theoretically preserving the $O(L)$ efficiency of standard ARMs without the computational overhead of block-wise vectorization.

However, strictly enforcing a causal constraint within a diffusion framework introduces a unique pathological state we term Information Collapse, which makes naive implementation unstable. In bidirectional architectures (e.g., BERT or MDLM), every token attends to the full global sequence. Even if a local region is heavily masked, the model can anchor its predictions on future tokens, maintaining a relatively uniform information density across positions. In contrast, under a causal mask, the visible context for a token $x_n$ is strictly limited to its predecessors $\mathbf{x}_{<n}$. This creates a severe information asymmetry: early tokens with short histories are extremely vulnerable to corruption. For instance, if the first few tokens of a sequence are masked, predicting the subsequent token becomes mathematically equivalent to random guessing, as there is neither past history nor future context to rely on. Conversely, later tokens in long sequences often possess redundant history and remain predictable even under moderate noise.

Standard uniform diffusion strategies ignore this asymmetry, treating the blind guessing scenarios of early tokens equally with the well-supported predictions of later tokens. Forcing the model to minimize loss on these invalid contexts results in high-variance gradients and optimization instability. To make CARD effective, we must explicitly address this variable reliability of the causal context. We propose two complementary strategies: \textit{Soft Tail Masking} (Section~\ref{subsec:noise_preference}) to structurally guarantee that the historical context retains valid signals, and \textit{Context-aware Reweighting} (Section~\ref{subsec:daum}) to adaptively down-weight predictions where the context remains too ambiguous.

\subsection{Soft Tail Masking}
\label{subsec:noise_preference}

Causal diffusion requires a noise strategy that respects the autoregressive nature of the model. Standard uniform masking is ill-suited here because it randomly corrupts tokens anywhere, including the sequence start ($n \ll L$). Since early tokens inherently possess little history, masking their few available context tokens effectively forces the model to predict from pure noise.
To guarantee a valid historical context, a natural intuition is to concentrate all corruption at the sequence tail. This maximizes the clean prefix, ensuring stable supervision. However, strict tail masking completely removes the immediate neighbors of the corrupted tokens, ignoring the strong local dependencies required for language modeling~\cite{distance}.

We propose \textit{Soft Tail Masking} (Figure~\ref{fig:CARD_DETAIL}), a strategy designed to alleviate the issue by restricting masking to a dynamic tail window
$[\max(0, L - \lambda t \cdot L), L]$. 
By maintaining a clean prefix while creating a mixed-state transition zone at the tail, we ensure the model accesses sufficient global history while retaining the local context needed for prediction. 
We prove in Appendix~\ref{appendix:foundations} (Proposition 2) that this preserves a higher lower bound on Mutual Information than uniform masking.

\subsection{Context-aware Reweighting}
\label{subsec:daum}
Compared to standard ARMs, CARD predicts $x_n$ from a stochastically corrupted prefix $\mathbf{x}_{<n}$. When the prefix is heavily masked, the conditional entropy $H(x_n|\mathbf{x}_{<n}^t)$ increases sharply. Forcing the model to produce confident predictions under such high uncertainty results in noisy gradients and optimization instability.
The diffusion models typically employ a global weighting scheme (e.g., $1/t$ in MDLM) to balance contributions across noise levels at sequence level, grounded in the ELBO framework. However, global weighting is insufficient for causal models since the effective noise level varies \textit{locally} at each token position $n$.

% To address this, 
Considering the causal characteristics of CARD, we introduce a context-aware reweighting mechanism. 
Specifically, we propose to evaluate the ambiguity of the context $\mathbf{x}_{<n}^t$ along three dimensions: Quantity (total noise count), Distance (proximity of noise to target), and Density (consecutive corruption).
These factors are synthesized into a unified local ambiguity score $S_n^{local}$, defined as the distance-weighted sum of corruption costs in the history:
% \begin{equation}
\label{eq:local_score_combined}
\begin{align}
    S_n^{local} &= \sum_{i=1}^{n} C_i \cdot (1-p)^{(n-i)}, \\
    % \text{where } 
    C_i &= \mathbb{I}[x_i = \texttt{[MASK]}] \cdot \left(1 + \mathbb{I}[x_{i-1} = \texttt{[MASK]}]\right).
\end{align}
% \end{equation}
The formulation explicitly maps the three dimensions to mathematical components:
\begin{itemize}[leftmargin=*]
    \item \textbf{Noise Quantity:} The summation $\sum_{i=1}^{n}$ accumulates the corruption costs across the history. This term ensures that a higher total number of masked tokens leads to a larger cumulative score, naturally suppressing the weight for heavily corrupted contexts.
    \item \textbf{Noise Distance:} Following previous findings that the relevance of historical tokens decays exponentially with distance~\cite{distance,exponentialdecay}, we introduce the decay factor $(1-p)^{(n-i)}$, where $p$ is a decay factor set to a constant 0.5. It ensures that noise in the immediate context will be penalized more heavily than noise in the distant past, as the immediate context is most critical for next-token prediction.
    \item \textbf{Noise Density:} The cost term $C_i$ assigns a higher cost to consecutive masked tokens (e.g., spans), reflecting the difficulty of reconstructing regions where local dependencies are entirely severed.
\end{itemize}

Finally, the context-aware loss weight $w(n, \mathbf{x}^t_{<n})$ is computed as:
\begin{equation}
\label{eq:daum_weight}
    w(n, \mathbf{x}^t_{<n}) = \frac{1}{\beta + S_n^{local}},
\end{equation}
where $\beta$ is a smoothing constant (typically set to 1). 

Our mechanism shifts the reweighting granularity from the sequence level (as in MDLM and BD3LM) to the token level. By down-weighting tokens in degraded contexts, the model focuses on regions with sufficient signal, leading to more efficient optimization (see Appendix~\ref{appendix:foundations}).

\begin{table*}[t]
\centering
\caption{LM Evaluation Harness results. All models are 1B parameters trained on 300B tokens.}
\label{tab:main-results}
\small 
\resizebox{\textwidth}{!}{ 
\begin{tabular}{l|cccccccc|c}
\toprule
\textbf{Model} & 
\begin{tabular}[c]{@{}c@{}} \footnotesize ARC-Challenge \\ \scriptsize 25-shot \end{tabular} & 
\begin{tabular}[c]{@{}c@{}} \footnotesize ARC-Easy \\ \scriptsize 25-shot \end{tabular} & 
\begin{tabular}[c]{@{}c@{}} \footnotesize CommonsenseQA \\ \scriptsize 7-shot \end{tabular} & 
\begin{tabular}[c]{@{}c@{}} \footnotesize HellaSwag \\ \scriptsize 3-shot \end{tabular} & 
\begin{tabular}[c]{@{}c@{}} \footnotesize MMLU-redux \\ \scriptsize 5-shot \end{tabular} & 
\begin{tabular}[c]{@{}c@{}} \footnotesize PIQA \\ \scriptsize 0-shot \end{tabular} & 
\begin{tabular}[c]{@{}c@{}} \footnotesize SciQ \\ \scriptsize 0-shot \end{tabular} & 
\begin{tabular}[c]{@{}c@{}} \footnotesize Winogrande \\ \scriptsize 5-shot \end{tabular} & 
\textbf{AVG} \\ 

\midrule
\rowcolor{gray!20}
\multicolumn{10}{c}{\textit{Autoregressive Models}} \\
\midrule
ARM   & 34.04          & 64.65          & 52.74          & 61.26          & 25.45          & 75.95          & 81.10          & 55.96          & 56.39 \\

\midrule
\rowcolor{gray!20}
\multicolumn{10}{c}{\textit{Diffusion Models}} \\
\midrule
BD3LM & 27.30          & 48.06          & 44.06          & 42.48          & \textbf{26.93} & 59.79          & 79.60          & 51.38          & 47.45 \\
MDLM  & 29.44          & 49.16          & 36.45          & 48.32          & 26.36          & 59.63          & 76.60          & \textbf{54.46} & 47.55 \\
CARD (Ours)  & \textbf{32.68} & \textbf{60.77} & \textbf{48.73} & \textbf{53.29} & 25.65          & \textbf{71.71} & \textbf{79.80} & 53.28          & \textbf{53.23} \\
\bottomrule
\end{tabular}
}
\end{table*}

\begin{table*}[t]
\centering
\caption{PPL evaluation on various text domains. Lower is better.}
\label{tab:ppl-comprehensive}
\small
\begin{tabularx}{\textwidth}{l YYYYYYYY c}
\toprule
\textbf{Model} & 
\textbf{AG News} & 
\textbf{arXiv} & 
\textbf{LAMBADA} & 
\textbf{LM1B} & 
\textbf{OpenWebText} & 
\textbf{PTB} & 
\textbf{PubMed} & 
\textbf{WikiText} & 
\textbf{AVG} \\ 

\midrule
\rowcolor{gray!20}
\multicolumn{10}{c}{\textit{Autoregressive Models}} \\
\midrule
ARM   & 30.62 & 18.15 & 33.83 & 39.14 & 17.68 & 117.56 & 11.93 & 40.52 & 38.68 \\

\midrule
\rowcolor{gray!20}
\multicolumn{10}{c}{\textit{Diffusion Models}} \\
\midrule
BD3LM & 41.18 & 44.60 & 39.17 & 40.04 & 40.97 & 118.30 & 34.66 & 39.28 & 49.78 \\
MDLM  & 42.20 & 23.58 & 35.87 & 48.68 & 20.77 & 168.19 & 17.23 & 42.18 & 49.84 \\
CARD (Ours) & \textbf{27.67} & \textbf{20.34} & \textbf{30.36} & \textbf{29.61} & \textbf{17.59} & \textbf{97.74} & \textbf{13.20} & \textbf{38.67} & \textbf{34.40} \\
\bottomrule
\end{tabularx}
\end{table*}

\subsection{Confidence-Based Block Inference}
\label{subsec:inference}

We employ a confidence-based block sampling strategy to accelerate generation. Specifically, at each generation step, we initialize a candidate block of length $K$ by appending mask tokens to the sequence tail, denoted as $\mathbf{x}^{(0)} = \{ \texttt{[MASK]}_1, \dots, \texttt{[MASK]}_K \}$. We then perform iterative parallel denoising, where a token $x_i$ at iteration $j$ is updated only if its prediction probability exceeds a threshold $\tau$~\cite{fastdllm,sdar}:
\begin{equation}
    x_i^{(j+1)} = 
    \begin{cases} 
    x_i^{(j)} & \text{if } x_i^{(j)} \neq \texttt{[MASK]}, \\
    \arg\max_{w} p_\theta(w | \mathbf{x}^{(j)}) & \text{if } \max_{w} p_\theta(w | \mathbf{x}^{(j)}) > \tau, \\
    \texttt{[MASK]} & \text{otherwise.}
    \end{cases}
\end{equation}
To strictly bound latency, we impose a maximum step limit $T_{max}$. If the block is not fully denoised within $T_{max}$ steps, all remaining masks are immediately decoded. Finally, the generated block is added to the KV cache. This approach allows the inference speed to be dynamically controlled by adjusting the block size $K$, threshold $\tau$, and step limit $T_{max}$.

\section{Experiments}
\label{sec:experiments}

To validate the effectiveness of the CARD framework, we benchmarked it against three architectures: ARM, MDLM, and BD3LM. All models were pre-trained on a 300B-token subset of FineWeb~\cite{fineweb} and aligned to a 1B-parameter scale. To ensure a fair comparison, the baselines utilized state-of-the-art optimizations: MDLM adopted variable-length packed QKV operators from Flash Attention~\cite{fa2}, while BD3LM integrated \texttt{torch.compile} with Flex Attention~\cite{flexattention}.  Detailed model structure configurations and training hyperparameters are provided in Appendix~\ref{app:experiment_setup}.

\subsection{Computational Efficiency}
We first address the training cost bottleneck typical of diffusion models. Normalizing the training latency of ARM and CARD to a baseline of $1.0\times$, MDLM incurs a $1.5\times$ cost due to its bidirectional attention mechanism, while BD3LM rises to roughly $3.0\times$ driven by input duplication constraints. In contrast, CARD eliminates these overheads, achieving superior performance while maintaining ARM-level training efficiency.

\subsection{Performance Evaluation}

\paragraph{Downstream Task Accuracy.}
We assessed disciplinary knowledge using ARC-Challenge \& ARC-Easy~\cite{arc}, MMLU~\cite{mmlu}, and SciQ~\cite{sciq}; commonsense reasoning via PIQA~\cite{piqa}, HellaSwag~\cite{hellaswag}, and CommonsenseQA~\cite{commonsenseqa}; and context disambiguation with Winogrande~\cite{WinoGrande}. As detailed in Table~\ref{tab:main-results}, a distinct performance hierarchy is evident. While the baseline methods MDLM and BD3LM plateau at an average of approximately 47.50\%, CARD establishes significantly better results for non-autoregressive models with an average accuracy of 53.23\%. The substantial 5.7\% absolute improvement over prior diffusion baselines indicates that CARD effectively mitigates the performance degradation. Crucially, while a marginal gap to the autoregressive (ARM) upper bound remains, CARD significantly narrows this disparity, demonstrating that dense supervision can yield ARM-competitive performance without sacrificing the efficiency benefits of parallel decoding.

\paragraph{Language Modeling and Generalization.}
To evaluate intrinsic generative quality, we measured zero-shot perplexity across three distinct domains: general corpora using WikiText~\cite{wikitext} and OpenWebText~\cite{openwebtext}; news and periodicals using AG News~\cite{agnews}, LM1B~\cite{lm1b}, and PTB~\cite{ptb}; and specialized or long-context tasks using arXiv, PubMed~\cite{scientificpaper}, and LAMBADA~\cite{LAMBADA}.
The results in Table~\ref{tab:ppl-comprehensive} show that CARD consistently outperforms both diffusion baselines. More notably, CARD surpasses the ARM baseline on 6 out of 8 datasets. It achieves the best overall scores on general domains and the context-heavy LAMBADA benchmark, while remaining competitive on the specialized scientific vocabularies of arXiv and PubMed.
We attribute the generalization advantage to the training objective. Standard ARMs rely on next-token prediction, which has been argued to be ``myopic,'' prioritizing local correlations and rote memorization~\cite{rollthedice}. Conversely, CARD's denoising objective functions as a form of ``teacherless training,'' forcing the model to predict tokens from corrupted contexts. The mechanism incentivizes the model to capture global structural patterns and long-range dependencies rather than relying on local statistical shortcuts, resulting in superior generalization on unseen data compared to the strictly left-to-right of ARMs.

\begin{table}[t]
\centering
% 修改点 1：明确 PPL 实验的模型大小和数据量
\caption{Perplexity (PPL) results on the LM1B dataset. Models are 110M parameters trained on 33B tokens using EMA.}
\label{tab:lm1b-ppl}
\small
\begin{tabularx}{\columnwidth}{l|YYYY} 
\toprule
\rowcolor{gray!20}
\textbf{Model} & \textbf{ARM} & \textbf{BD3LM} & \textbf{MDLM} & \textbf{CARD} \\ \midrule
PPL $\downarrow$ & \textbf{21.12} & 35.06 & 37.48 & 21.54 \\ 
\bottomrule
\end{tabularx}
\end{table}

\begin{figure}[t] % h:当前位置, t:页顶, b:页底, p:独立一页
    \centering
    \includegraphics[width=\columnwidth]{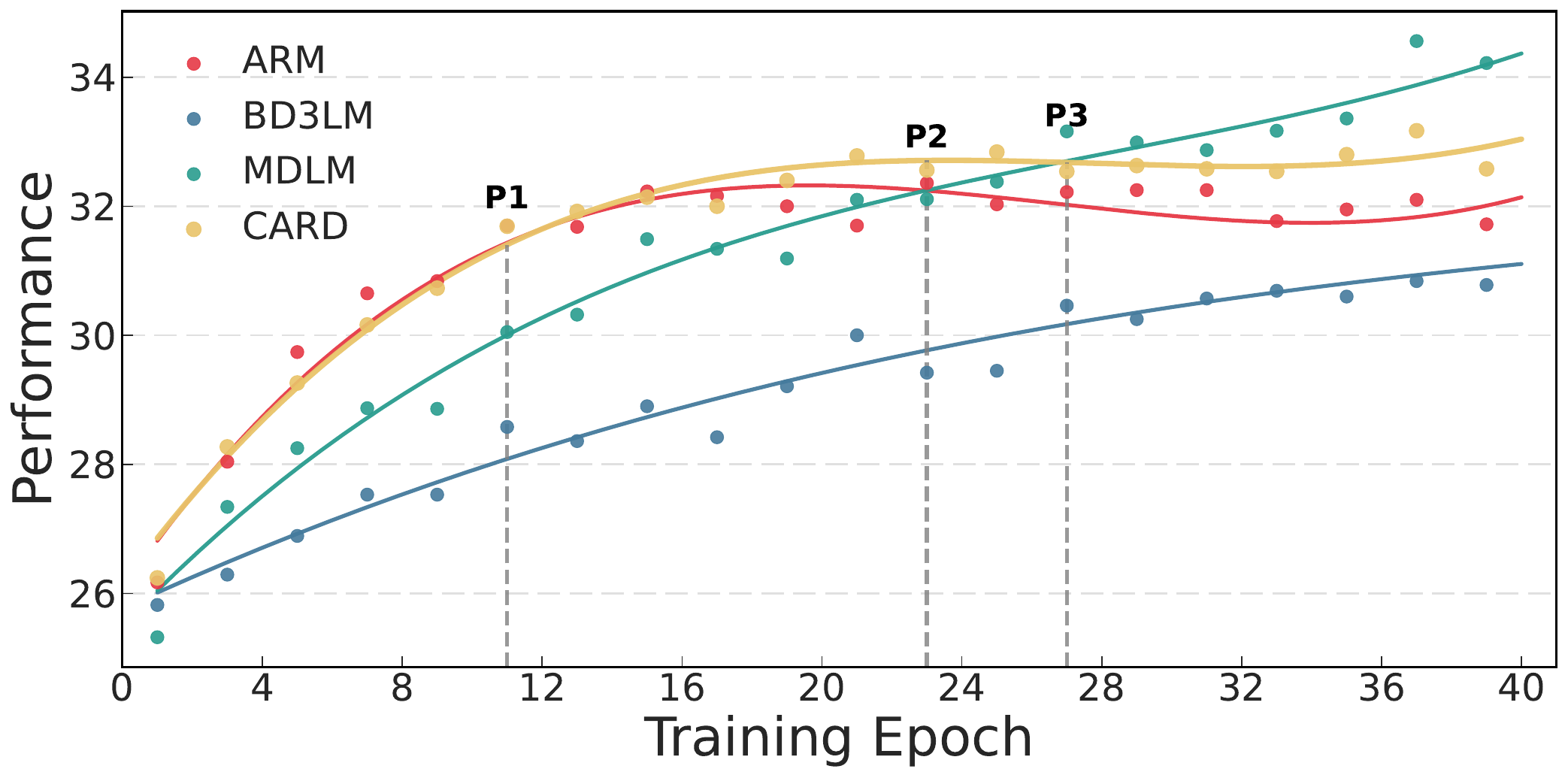} 

    \caption{HellaSwag performance of four paradigms under repeated training on a FineWeb-Edu subset. The annotations mark specific crossover points in performance: P1 denotes the epoch where CARD surpasses ARM, P2 where MDLM overtakes ARM, and P3 where MDLM exceeds CARD.}
    \label{fig:epoch_scaling}
\end{figure}

\begin{table*}[t]
\centering

\caption{Ablation study on noise position and context-aware reweighting mechanisms.}

\label{tab:ablation}
\small 
\resizebox{\textwidth}{!}{ 
\begin{tabular}{l|cccccccc|c}
\toprule
\textbf{Setting} & 
\begin{tabular}[c]{@{}c@{}} \footnotesize ARC-C \\ \scriptsize 25-shot \end{tabular} & 
\begin{tabular}[c]{@{}c@{}} \footnotesize ARC-E \\ \scriptsize 25-shot \end{tabular} & 
\begin{tabular}[c]{@{}c@{}} \footnotesize CSQA \\ \scriptsize 7-shot \end{tabular} & 
\begin{tabular}[c]{@{}c@{}} \footnotesize HellaS \\ \scriptsize 3-shot \end{tabular} & 
\begin{tabular}[c]{@{}c@{}} \footnotesize MMLU \\ \scriptsize 5-shot \end{tabular} & 
\begin{tabular}[c]{@{}c@{}} \footnotesize PIQA \\ \scriptsize 0-shot \end{tabular} & 
\begin{tabular}[c]{@{}c@{}} \footnotesize SciQ \\ \scriptsize 0-shot \end{tabular} & 
\begin{tabular}[c]{@{}c@{}} \footnotesize Wino \\ \scriptsize 5-shot \end{tabular} & 
\textbf{AVG} \\ 

\midrule
\textbf{CARD (Ours)} & \textbf{32.68} & \textbf{60.77} & \textbf{48.73} & \textbf{53.29} & \textbf{25.45} & \textbf{71.71} & \textbf{79.80} & \textbf{53.28} & \textbf{53.21} \\

\midrule
\multicolumn{10}{l}{\textit{Ablation on Noise Position}} \\
\quad w/o Relaxed Window (Strict Tail) & 32.85 & 59.72 & 48.24 & 52.33 & 25.38 & 71.28 & 78.90 & 51.31 & 52.50 \\
\quad w/o Tail Preference (Random)     & 29.95 & 59.38 & 46.03 & 51.78 & 25.14 & 71.00 & 76.30 & 53.36 & 51.62 \\

\midrule
\multicolumn{10}{l}{\textit{Ablation on Weighting}} \\
\quad w/o Context-aware Weighting            & 30.63 & 58.29 & 47.01 & 50.14 & 25.23 & 70.35 & 79.20 & 52.41 & 51.66 \\
\bottomrule
\end{tabular}
}
\end{table*}

\section{Analysis}

\subsection{Training Stability}
\label{sec:stability}

A bottleneck in scaling discrete diffusion models is optimization instability. As derived in Appendix~\ref{appendix:foundations} (Proposition 3), BD3LM suffers from distributional discontinuities at block boundaries, while MDLM encounters high variance when predicting tokens from heavily masked contexts. To address these issues, practical implementations of these architectures often rely heavily on Exponential Moving Average (EMA). Although rarely highlighted in their theoretical formulations, EMA is adopted by default in the official training repositories of both MDLM and BD3LM as a necessary stabilizer. In contrast, CARD is designed for inherent stability through its continuous causal loss landscape and context-aware reweighting (Proposition 1), which minimizes gradient variance by design. This theoretical guarantee reduces the dependency on aggressive parameter smoothing, ensuring that the optimization trajectory remains true to the underlying data distribution.

\textbf{Empirical Validation with EMA.} 
To investigate the training stability in detail, we conducted a controlled study on the LM1B dataset. We trained all models (110M parameters, 33B tokens) using the EMA configurations explicitly found in the baseline codebases. As shown in Table~\ref{tab:lm1b-ppl}, even with EMA effectively buffering the gradient noise, MDLM and BD3LM yield perplexities of 37.48 and 35.06, respectively. In comparison, CARD achieves a significantly lower perplexity of 21.54 under identical conditions. The result demonstrates that while baselines require EMA to manage their structural instability, CARD utilizes it to further refine its density estimation, converging to a solution comparable to the ARM baseline.

\subsection{Data Potential and Epoch Scaling}
\label{subsec:data_potential}

We define \textit{Data Potential} as an architecture's capacity to continuously extract signal from a fixed data distribution over repeated training iterations. Theoretically, based on the number of learnable conditional probability paths per datum (derived in Appendix~\ref{appendix:complexity}), we posit a hierarchy of $\text{MDLM} > \text{CARD} > \text{BD3LM} > \text{ARM}$, suggesting that ARMs saturate rapidly, whereas diffusion-based models sustain gains over longer horizons.

Empirical validation on 1B-parameter models trained on a 1B token subset of FineWeb-Edu~\cite{fineweb} over 40 epochs confirms the ranking (Figure~\ref{fig:epoch_scaling}). At the early training stage, CARD surpasses the ARM baseline at the inflection point P1 ($\approx$ epoch 11) as the latter saturates. MDLM later overtakes ARM (P2) and eventually CARD (P3). Crucially, the interval preceding P3 identifies a functional ``sweet spot'': CARD significantly outperforms ARM without requiring the extensive training horizon MDLM needs to realize its full potential.

The result has critical implications given the current scarcity of high-quality data, which necessitates training beyond Chinchilla-optimal ratios to minimize inference costs. 
While standard ARMs are ill-suited for the regime due to early saturation and MDLMs incur high initial compute costs, CARD effectively bridges the gap. It extends the performance boundary within practical computational budgets, offering a superior scaling solution when data quantity is the primary bottleneck.

\begin{table}[t]
\centering
\caption{Gen PPL results.}
\label{tab:gen-ppl}
\resizebox{\columnwidth}{!}{ 
\begin{tabular}{lrc} % 将中间的 c 改为 r
\toprule
\textbf{Decoding Configuration} & \multicolumn{1}{c}{\textbf{Throughput (tok/s)}} & \textbf{AVG PPL $\downarrow$} \\ 
\midrule
\textbf{ARM (Baseline)} & 3,771 \scriptsize{(1.00$\times$)} & 11.19 \\
\midrule
\multicolumn{3}{l}{\textit{\textbf{CARD (Ours)}}} \\
\hspace{3mm} Block=16, Steps=16 & 6,441 \scriptsize{(1.71$\times$)} & 12.65 \\
\hspace{3mm} Block=16, Steps=8  & 10,702 \scriptsize{(2.84$\times$)} & 13.81 \\
\hspace{3mm} Block=32, Steps=8  & 15,064 \scriptsize{(4.01$\times$)} & 18.38 \\
\bottomrule
\end{tabular}
} 
\end{table}

\subsection{Ablation Study}
\label{subsec:ablation}

In addition to the architectural comparison, we conducted ablation studies to validate the effectiveness of our proposed noise position preference (Section \ref{subsec:noise_preference}) and context-aware reweighting mechanisms (Section \ref{subsec:daum}). The results are summarized in Table~\ref{tab:ablation}, leading to the following observations.

\paragraph{The noise distribution strategy plays a crucial role in unidirectional models.}
As shown in the results, applying noise to random positions (w/o Tail Preference) yields the lowest performance among the noise strategies. In a causal framework, tokens at the beginning of the sequence lack preceding context. If these tokens are masked randomly, the model cannot recover them effectively, leading to training inefficiencies. By concentrating noise at the tail, we observe a clear performance improvement. This suggests that a tail-biased noise strategy better aligns with the generative nature of language modeling, where history is used to predict the future. Furthermore, the results highlight the importance of the relaxed noise window. The ``Strict Tail'' setting, where the end of the sequence is a solid block of noise, underperforms compared to the full CARD implementation. A solid noise block creates an information void where the final tokens lack any immediate local context. By allowing a mix of clean and noisy tokens within the tail window (Relaxed Window), 
% we prevent the formation of this information black hole, 
we enable the model to leverage local cues even during the denoising process.

\paragraph{Removing context-aware reweighting results in a noticeable drop in accuracy across most benchmarks.}
The dynamic weighting mechanism, rooted in the ELBO formulation, uses noise intensity to balance the training objective. It naturally integrates the next-token prediction task with the diffusion objective by assigning appropriate importance to each token based on the clarity of its context. This ensures that the model focuses on learnable patterns rather than being overwhelmed by high-entropy predictions in heavily corrupted contexts.

\subsection{Generation Perplexity Analysis}

To further evaluate the generation quality, we conducted a generation perplexity (Gen PPL) analysis on Hellaswag prefixes using the model trained in our main experiment. For robust evaluation, we report the average PPL computed by four base models: \texttt{Qwen3-8B}~\cite{qwen3}, \texttt{SmolLM3-3B}~\cite{smollm3}, \texttt{gemma-3-27b}~\cite{gemma3}, and \texttt{gpt2-large}~\cite{gpt2}. All inference tests were performed with a batch size of 128. As shown in Table \ref{tab:gen-ppl}, our method demonstrates a promising trade-off between speed and quality. Specifically, we achieve a 1.62$\times$ speedup while maintaining a generation quality comparable to the ARM baseline. Furthermore, in a more aggressive setting, our method delivers over 4$\times$ inference acceleration with only a slight increase in PPL. These results strongly validate the potential of our method to serve as a new baseline for efficient generation. Additionally, we provide a detailed case study and discuss the potential failure modes of parallel generation in Appendix~\ref{app:case_study}.

\section{Conclusion}

We presented CARD, a unified framework that reconciles the training stability of autoregressive models with the parallel inference capabilities of diffusion. By reformulating discrete diffusion within a strict causal constraint, CARD eliminates the computational overhead of block-based architectures. Empirically, CARD not only matches the generation quality of standard ARMs but also speed up to $1.7\times$ through dynamic parallel decoding. Crucially, our analysis of data potential reveals that CARD avoids early saturation in multi-epoch regimes, positioning it as a highly data-efficient backbone for next-generation LLMs.

\section*{Impact Statement}

This paper presents work whose goal is to advance the field of Machine Learning, particularly by improving the training and inference efficiency of Large Language Models. There are many potential societal consequences of our work, none which we feel must be specifically highlighted here.

\bibliography{example_paper}
\bibliographystyle{icml2026}

\newpage
\appendix
\onecolumn

\section{Mathematical Foundations of CARD}
\label{appendix:foundations}
In this section, we provide a formal analysis of the optimization dynamics and information-theoretic properties of the Causal Autoregressive Diffusion (CARD) framework. We contrast CARD with Masked Discrete Diffusion Models (MDLM) and Block-wise Discrete Diffusion Models (BD3LM).

\subsection{Notation and Preliminaries}

Let $\mathbf{x} = (x_1, \dots, x_L)$ be a sequence of length $L$ from a discrete vocabulary $\mathcal{V}$. Let $\mathcal{M} \subset \{1, \dots, L\}$ denote the set of indices masked at time $t \in [0,1]$. For any position $n$, we define the \textit{causal context} $\mathcal{C}_n = \{x_i \mid i < n, i \notin \mathcal{M}\}$. The training objective is to minimize the expected negative log-likelihood:
\begin{equation}
\mathcal{L}(\theta) = \mathbb{E}_{t, \mathcal{M}} \left[ \sum_{n=1}^L w(n, \mathcal{C}_n) \cdot \ell_n(\theta; \mathcal{C}_n) \right]
\end{equation}
where $\ell_n(\theta; \mathcal{C}_n) = -\log p_\theta(x_n \mid \mathcal{C}_n)$ and $w(n, \mathcal{C}_n)$ is the weight assigned to the prediction at position $n$.

\begin{definition}[Local Ambiguity Score]
The Local Ambiguity Score $S_n^{local}$ is defined as a weighted sum of corruption costs within the causal window:
\begin{equation}
S_n^{local}(\mathcal{C}_n) = \sum_{i=1}^{n-1} C_i \cdot (1-p)^{n-i}
\end{equation}
where $C_i = \mathbb{I}[i \in \mathcal{M}] \cdot (1 +  \mathbb{I}[i-1 \in \mathcal{M}])$ represents the cost of masking, and $p \in (0,1)$ is a decay factor.
\end{definition}

\subsection{Proposition 1: Gradient Variance Stabilization}

\begin{proposition}
The CARD weighting scheme $w(n, \mathcal{C}_n) = (\beta + S_n^{local})^{-1}$ minimizes the variance of the stochastic gradient estimator by performing an instance-level inverse-variance weighting.
\end{proposition}

\begin{proof}
Consider the variance of the stochastic gradient $\mathbf{g}_n = \nabla_\theta \ell_n$. In the discrete diffusion setting, as the context $\mathcal{C}_n$ becomes increasingly corrupted (high $S_n^{local}$), the conditional distribution $p_\theta(x_n \mid \mathcal{C}_n)$ approaches the uninformative marginal distribution $p(x_n)$. In this regime, the Fisher Information $\mathcal{I}(\theta)_{n} = \mathbb{E}[\nabla_\theta \ell_n \nabla_\theta \ell_n^\top]$ is dominated by the noise of the sampling process rather than the underlying structural signal of the language.

Let $\sigma_n^2(\mathcal{C}_n) = \|\nabla_\theta \ell_n(\mathcal{C}_n)\|^2$ be the squared norm of the gradient. Given the power-law decay of mutual information in sequences, we posit that $\sigma_n^2$ is monotonically bounded by the ambiguity score: $\sigma_n^2 \leq \alpha S_n^{local} + \epsilon$.
The variance of the weighted estimator is:
\begin{equation}
\text{Var}[w \cdot \mathbf{g}_n] = \mathbb{E}[w^2 \|\mathbf{g}_n\|^2] - \|\mathbb{E}[w \mathbf{g}_n]\|^2 \leq \frac{\alpha S_n^{local} + \epsilon}{(\beta + S_n^{local})^2}
\end{equation}
As $S_n^{local} \to \infty$, the weighted gradient norm $\|w \mathbf{g}_n\| \to 0$. This ensures that uninformative, high-entropy contexts do not contribute disproportionately to the parameter updates, satisfying the conditions for stable convergence in the absence of aggressive Exponential Moving Average (EMA).
\end{proof}

\subsection{Proposition 2: Signal Retention via Causal MI Maximization}

\begin{proposition}
For a fixed noise budget $t$, the Soft Tail Masking strategy preserves a strictly higher lower bound on the cumulative Mutual Information (MI) compared to Uniform Masking.
\end{proposition}

\begin{proof}
Let $I(x_n; x_i)$ be the MI between tokens. In natural language, $I(x_n; x_i) \approx f(|n-i|)$, where $f$ is a monotonically decreasing function. The total information available to the model is $\mathcal{I}_{total} = \sum_{n=1}^L \sum_{i < n, i \notin \mathcal{M}} I(x_n; x_i)$.

1. \textbf{Uniform Masking:} For MDLM, each $i \in \mathcal{M}$ with probability $t$. The expected MI at position $n$ is $(1-t) \sum_{i < n} I(x_n; x_i)$.
2. \textbf{Soft Tail Masking:} CARD restricts masks to the tail window. For $n < L(1-\lambda t)$, the probability $P(i \in \mathcal{M} \mid i < n) = 0$.

Since $I(x_n; x_i)$ is maximal when $n-i$ is small, the Soft Tail strategy ensures that for a significant portion of the sequence (the ``Head''), the model observes the full causal signal. Because $\sum_{n=1}^L I(x_n; \mathcal{C}_n^{CARD})$ prioritizes preserving low-distance dependencies which contain the highest MI, it follows that $\mathcal{I}_{total}^{CARD} > \mathcal{I}_{total}^{MDLM}$.
\end{proof}

\subsection{Proposition 3: Landscape Continuity and Block Discontinuity}

\begin{proposition}
CARD eliminates the $O(1)$ distributional shift discontinuities present in block-wise diffusion architectures (BD3LM).
\end{proposition}

\begin{proof}
Let $\mu_n$ be the distribution of the context $\mathcal{C}_n$. We evaluate the continuity of the loss landscape by the Total Variation (TV) distance between adjacent context distributions $d_{TV}(\mu_n, \mu_{n+1})$.

In \textbf{BD3LM}, sequences are partitioned into blocks $\{B_k\}$. At a boundary index $j$ where $x_j \in B_k$ and $x_{j+1} \in B_{k+1}$, the context shifts from a deterministic clean history (from previous blocks) to a stochastic noisy context (within the current block). This implies:
\begin{equation}
\lim_{L \to \infty} d_{TV}(\mu_j, \mu_{j+1}) = \| p(x_{clean}) - p(x_{noisy}) \|_{TV} \approx \mathcal{O}(1)
\end{equation}
This jump results in a non-Lipschitz gradient spike at every block boundary.

In \textbf{CARD}, the transition probability $P(x_n = \text{[MASK]})$ is defined by a continuous noise schedule $\sigma(n, t)$ over the sequence index. For a linear schedule, the change in masking probability between $n$ and $n+1$ is $O(1/L)$. Thus, $d_{TV}(\mu_n, \mu_{n+1}) \leq \frac{K}{L}$, ensuring that the expected loss and its gradients are Lipschitz continuous with respect to the sequence index.
\end{proof}

\section{Experimental Setups}
\label{app:experiment_setup}
We detail the model architecture and training hyperparameters used in our experiments, with the full configuration summarized in Table~\ref{tab:experimental_setup}.

\paragraph{Model Architecture}
Our model is built upon a bidirectional Transformer encoder architecture, incorporating Flash Attention 2 for computational efficiency. It consists of 33 Transformer layers with a hidden dimension of 1536 and an intermediate FFN dimension of 4096, utilizing the SiLU activation function. The model supports a maximum position embedding length of 8192 tokens.

\paragraph{Training Configuration}
Training is performed using the AdamW optimizer with \texttt{bfloat16} mixed precision. We employ a constant learning rate schedule with a 2,500-step warmup, peaking at $3 \times 10^{-4}$. For the diffusion process, the masking probability is linearly annealed from 1.0 to 0.

\begin{table}[h]
    \centering
    \caption{Experimental Setup: Model Architecture and Training Hyperparameters}
    \label{tab:experimental_setup}
    \vspace{2mm}
    % 定义列格式：左表两列 + 中间间隔 + 右表两列
    % p{0.5cm} 用来制造中间的空白间隔
    \begin{tabular}{lc p{0.5cm} lc}
        \toprule
        % --- 表头 ---
        \multicolumn{2}{c}{\textbf{Model Architecture}} & & \multicolumn{2}{c}{\textbf{Training Hyperparameters}} \\
        \cmidrule(r){1-2} \cmidrule(l){4-5} % 局部横线，中间断开
        \textbf{Parameter} & \textbf{Value} & & \textbf{Hyperparameter} & \textbf{Value} \\
        \midrule
        % --- 数据行（一一对应） ---
        Number of Layers & 33 & & Optimizer & AdamW \\
        Hidden Size & 1536 & & Peak Learning Rate & $3 \times 10^{-4}$ \\
        Intermediate Size & 4096 & & LR Scheduler & Cosine w/ Warmup \\
        Attention Heads & 24 & & Warmup Steps & 2,500 \\
        Vocab Size & 50,368 & & Max Training Steps & 1,000,000 \\
        Activation Function & SiLU & & Sequence Length & 128 \\
        Max Pos Embeddings & 8,192 & & Precision & BF16 \\
        Attn Implementation & Flash Attn 2 & & Noise Schedule & Linear ($1.0 \to 0.0$) \\
        \bottomrule
    \end{tabular}
\end{table}

\section{Complexity Analysis of Learnable Conditional Probabilities}
\label{appendix:complexity}

In this section, we quantify the number of structural conditional probabilities that different generative models can learn. We define $L$ as the sequence length. We analyze the theoretical upper bound of dependency patterns based on the attention mechanism and the masking strategy employed by each model.

\subsection{Autoregressive Models (ARM)}
Standard Autoregressive Models rely on the probability chain rule. The generation of a token at position $t$ depends strictly on the fixed sequence of preceding tokens $x_{1}, \dots, x_{t-1}$. Since the context for every position is deterministic and unique (the prefix), the model does not learn from varying subsets of the context.
Therefore, the total number of learnable conditional probabilities is linear with respect to the sequence length:
\begin{equation}
    N_{\text{ARM}} = L
\end{equation}

\subsection{Causal Autoregressive Diffusion (CARD)}
CARD combines unidirectional attention with a discrete diffusion process. Although the attention mechanism restricts information flow from left to right, the noise injection process introduces combinatorial diversity. For a token at position $t$, the context consists of tokens $x_{1}$ to $x_{t-1}$. In the diffusion training process, each of these context tokens can exist in two states: masked or unmasked.

This results in a geometric series where the first token has 1 possible context state, the second has 2, and the last has $2^{L-1}$. The total number of combinations is the sum of this series:
\begin{equation}
    N_{\text{CARD}} = \sum_{t=0}^{L-1} 2^t = 2^L - 1
\end{equation}

\subsection{Masked Discrete Language Models (MDLM)}
MDLM represents the standard bidirectional discrete diffusion approach. The model utilizes bidirectional attention, allowing any token to attend to any other token in the sequence. During training, a random proportion of tokens are masked.

For any given target position $i$, the context is a subset of the remaining $L-1$ tokens. Since each of the other tokens can be either masked or unmasked, there are $2^{L-1}$ possible context configurations for a single position. Since all $L$ positions serve as prediction targets, the total number of learnable probabilities is:
\begin{equation}
    N_{\text{MDLM}} = L \times 2^{L-1}
\end{equation}

\subsection{Blockwise Diffusion (BD3LM)}
BD3LM employs a hybrid architecture. It divides the sequence of length $L$ into $N$ blocks, where each block has a size of $K$ (such that $L = N \times K$). The model applies unidirectional causal attention between blocks but maintains bidirectional attention within each block.

Since the inter-block connection is causal, previous blocks act as a fixed context and do not contribute to combinatorial explosion. However, within each block of size $K$, the model behaves like a bidirectional diffusion model. The number of combinations per block is $K \times 2^{K-1}$. Summing this over all $N$ blocks yields:
\begin{equation}
    N_{\text{BD3LM}} = \frac{L}{K} \times (K \times 2^{K-1}) = L \times 2^{K-1}
\end{equation}

\subsection{Summary}
Table \ref{tab:complexity_comparison} summarizes the number of learnable conditional probabilities for each model. This comparison highlights that while diffusion-based models offer exponentially larger state spaces than ARM, Blockwise Diffusion (BD3LM) effectively bridges the gap by controlling the exponent through the block size $K$.

\begin{table}[h]
    \centering
    \caption{Comparison of Learnable Conditional Probabilities}
    \label{tab:complexity_comparison}
    \begin{tabular}{lcc}
        \toprule
        \textbf{Model Type} & \textbf{Attention Type} & \textbf{Complexity} \\
        \midrule
        ARM & Unidirectional & $L$ \\
        CARD & Unidirectional & $2^L - 1$ \\
        BD3LM & Hybrid & $L \times 2^{K-1}$ \\
        MDLM & Bidirectional & $L \times 2^{L-1}$ \\
        \bottomrule
    \end{tabular}
\end{table}

\section{Case Study: Impact of Acceleration Ratios}
\label{app:case_study}
Table~\ref{tab:case_study_card} presents a qualitative comparison between the ARM baseline and our CARD method. 
At a moderate acceleration ratio of $1.7\times$, CARD maintains generation quality comparable to the baseline, producing coherent and contextually appropriate text. 
However, aggressively increasing the speedup to $4\times$ by restricting the step budget leads to noticeable degradation. 
Instead of syntactic errors, this degradation primarily manifests as \textbf{logical repetition and text looping} (e.g., repeating similar sentence structures or phrases). 
This phenomenon stems from the hard step limit: the model is compelled to complete the text block via non-autoregressive generation at the final step. 
Lacking sufficient autoregressive guidance, the model tends to collapse into high-probability repetitive patterns rather than developing diverse narrative progressions.
While we anticipate that stronger base models will mitigate this sensitivity, we currently recommend adhering to the standard configuration to strike the optimal balance between speed and quality.

% ----------------------------------------------------
% 把之前的 Table 代码放在这里引用
% ----------------------------------------------------
\begin{table*}[h]
\centering
\caption{Comparison of generation quality under different acceleration settings. \textbf{Case 1} demonstrates how aggressive speedup ($4\times$) leads to repetitive sentence structures. \textbf{Case 2} further illustrates it, where CARD ($4\times$) falls into a degenerative loop (repeating ``applying the gel... bottle is shown''), whereas the baseline and moderate settings maintain narrative flow.}
\small
% 定义一些颜色以提升美观度 (可选)
\definecolor{darkgreen}{RGB}{0, 100, 0}
\definecolor{darkred}{RGB}{139, 0, 0}
\definecolor{promptgray}{gray}{0.95}

\begin{tabularx}{0.98\textwidth}{l X}
\toprule
% ================= CASE 1 =================
\multicolumn{2}{c}{\textbf{\textit{Case Study 1}}} \\
\midrule
\rowcolor{promptgray}
\textbf{Prompt} & \textit{A man is sitting on a roof. he...} \\
\addlinespace[0.4em]

\textbf{\textcolor{darkgreen}{ARM}} & 
sees a beautiful young lady on a rooftop. 
The man: a woman? 
The woman: no. 
The man: Oh! Well, I'm sure she's beautiful. 
The woman: No, she's not... \\
\addlinespace[0.4em]

\textbf{CARD ($1.7\times$)} & 
's living a peaceful life, working hard to support his family and doing what he loves. But there are some days that he can't do it anymore. He's getting older, he's getting tired... \\
\addlinespace[0.4em]

\textbf{\textcolor{darkred}{CARD ($4\times$)}} & 
 is looking at the sky. he is thinking about his life. he is thinking about his past. he is thinking about his future... \\

\midrule
\midrule
% ================= CASE 2 (NEW) =================
\multicolumn{2}{c}{\textbf{\textit{Case Study 2}}} \\
\midrule
\rowcolor{promptgray}
\textbf{Prompt} & \textit{A bottle of deep cleansing gel is shown. a woman} \\
\addlinespace[0.4em]

\textbf{\textcolor{darkgreen}{ARM}} & 
's face is washed with a mild cleanser and she applies face lotion with aloe vera. a woman's face is washed with a mild cleanser and she applies face lotion with aloe vera. \\
\addlinespace[0.4em]

\textbf{CARD ($1.7\times$)} & 
is shown with legs raised. A bottle of liquid hand soap is shown. a woman is shown leaving the shop and a young man is seen patting her down. \\
\addlinespace[0.4em]

\textbf{\textcolor{darkred}{CARD ($4\times$)}} & 
is shown applying the gel to her face. the bottle of the gel is shown. the woman is shown applying the gel to her face. the bottle of the gel is shown. A bottle of deep cleansing gel is shown. a woman is shown applying the gel to her face. the bottle of the gel is shown... \\

\bottomrule
\end{tabularx}

\label{tab:case_study_card}
\end{table*}
\end{document}